\documentclass{article}
\usepackage{xr}
\usepackage{nips14submit_e,times}
\usepackage{hyperref}
\usepackage{url}
\usepackage{enumitem}

\usepackage{tikz}
\usepackage{pgfplots}
\pgfplotsset{compat=newest}
\usepackage{wrapfig}

\usepackage{amsthm,amsfonts,amsmath,amssymb,epsfig,color,float,graphicx,verbatim}
\usepackage{algorithm,algorithmic}

\newtheorem{theorem}{Theorem} 
\newtheorem{lemma}{Lemma} \newtheorem{corollary}{Corollary}

\newcommand{\nn}{\mathcal{N}}

\newcommand{\xx}{\mathbf{x}}
\newcommand{\Approx}{\mathrm{Approx}}
\newcommand{\ww}{\mathbf{w}}
\newcommand{\vv}{\mathbf{v}}
\newcommand{\uu}{\mathbf{u}}
\newcommand{\reals}{\mathbb{R}} 
 \newcommand{\sign}{\mathrm{sign}}

\newcommand{\dpn}{\mathcal{P}}

\newcommand{\ignore}[1]{}

\newcommand{\X}{\mathcal{X}}
\newcommand{\poly}{\mathrm{poly}}
\newcommand{\sig}{\mathrm{sig}}
\newcommand{\relu}{\mathrm{relu}}

\newcommand{\secref}[1]{Sec.~\ref{#1}}

\newcommand{\figref}[1]{Fig.~\ref{#1}}
\renewcommand{\eqref}[1]{Eq.~(\ref{#1})}
\newcommand{\lemref}[1]{Lemma~\ref{#1}}
\newcommand{\corref}[1]{Corollary~\ref{#1}}
\newcommand{\thmref}[1]{Thm.~\ref{#1}}

\newenvironment{Ouralgorithm}[1][\  ] %
{
\rm
\begin{tabbing}
.\=...\=...\=...\=...\=  \+ \kill
} %
{\end{tabbing}
}
\floatstyle{ruled}
\newfloat{Algorithm}{thp}{alg}
\newenvironment{Balgorithm} %
{
\begin{minipage}{1.0\linewidth}
\begin{Ouralgorithm} %
} { \end{Ouralgorithm} \end{minipage} }

\title{On the Computational Efficiency of Training Neural Networks}
\author{Roi Livni\\ The Hebrew University \\ roi.livni@mail.huji.ac.il
  \And Shai Shalev-Shwartz\\ The Hebrew University \\
  shais@cs.huji.ac.il \And
  Ohad Shamir\\
  Weizmann Institute of Science\\
  ohad.shamir@weizmann.ac.il } \date{}

\nipsfinalcopy 

\begin{document}

\maketitle

\begin{abstract}
  It is well-known that neural networks are computationally hard to train.
  On the other hand, in practice, modern day neural networks are trained
  efficiently using SGD and a variety of tricks that include different
  activation functions (e.g. ReLU), over-specification (i.e., train
  networks which are larger than needed), and regularization. In this paper
  we revisit the computational complexity of training neural networks from
  a modern perspective. We provide both positive and negative results, some
  of them yield new provably efficient and practical algorithms for
  training certain types of neural networks.
\end{abstract}

\section{Introduction}\label{sec:intro}

One of the most significant recent developments in machine learning has been the resurgence of ``deep learning'', usually in the form of artificial neural networks. A combination of algorithmic advancements, as well as increasing computational power and data size, has led to a breakthrough in the effectiveness of neural networks, and they have been used to obtain very impressive practical performance on a variety of domains (a few recent examples include
\cite{LRMDCCDN12,krizhevsky2012imagenet,zeiler2013visualizing,dahl2013dropout,bengio2013representation}).

A neural network can be described by a (directed acyclic) graph, where each vertex in the graph corresponds to a neuron and each edge is associated with a weight. Each neuron calculates a weighted sum of the outputs of neurons which are connected to it (and possibly adds a bias term). It then passes the resulting number through an activation function $\sigma : \reals \to \reals$ and outputs the resulting number. We focus on feed-forward neural networks, where the neurons are arranged in layers, in which the output of each layer forms the input of the next layer. Intuitively, the input goes through several transformations, with higher-level concepts derived from lower-level ones. The depth of the network is the number of layers and the size of the network is the total number of neurons.

From the perspective of statistical learning theory, by specifying a neural
network architecture (i.e. the underlying graph and the activation function)
we obtain a hypothesis class, namely, the set of all prediction rules
obtained by using the same network architecture while changing the weights of
the network. Learning the class involves finding a specific set of weights,
based on training examples, which yields a predictor that has good
performance on future examples. When studying a hypothesis class we are
usually concerned with three questions:
\begin{enumerate}
\item \vskip -0.2cm \emph{Sample complexity:} how many examples are
    required to learn the class.
\item \emph{Expressiveness:} what type of functions can be expressed by predictors in the class.
\item \emph{Training time}: how much computation time is required to learn the class.
\end{enumerate}
\vskip -0.2cm
For simplicity, let us first consider neural networks with a
threshold activation function (i.e. $\sigma(z) = 1$ if $z>0$ and $0$
otherwise), over the boolean input space, $\{0,1\}^d$, and with a single
output in $\{0,1\}$. The sample complexity of such neural networks is well
understood \cite{AnBa02}. It is known that the VC dimension grows linearly
with the number of edges (up to log factors). It is also easy to see that no
matter what the activation function is, as long as we represent each weight
of the network using a constant number of bits, the VC dimension is bounded
by a constant times the number of edges. This implies that empirical risk
minimization - or finding weights with small average loss over the training
data - can be an effective learning strategy from a statistical point of
view.

As to the expressiveness of such networks, it is easy to see that
neural networks of depth $2$ and sufficient size can express all
functions from $\{0,1\}^d$ to $\{0,1\}$. However, it is also possible
to show that for this to happen, the size of the network must be
exponential in $d$ (e.g. \cite[Chapter 20]{SSSSBD14}).
Which functions can we express using a network of polynomial size? The
theorem below shows that all boolean functions that can be calculated in
time $O(T(d))$, can also be expressed by a network of depth $O(T(d))$ and size
$O(T(d)^2)$.
\begin{theorem}\label{thm:turing}
  Let $T : \mathbb{N} \to \mathbb{N}$ and for every $d$, let $\mathcal{F}_d$ be the set of
  functions that can be implemented by a Turing machine using
  at most $T(d)$ operations. Then there exist constants $b,c \in
  \reals_+$ such that for every $d$, there is a network architecture of depth $c\,T(d)+b$, size of
  $(c\,T(d)+b)^2$, and threshold activation function, such that the resulting hypotesis class contains
  $\mathcal{F}_d$.
\end{theorem}
The proof of the theorem follows directly from the relation between the time
complexity of programs and their circuit complexity (see, e.g., \cite{SipserBook}), and the fact
that we can simulate the standard boolean gates using a fixed number of neurons.

We see that from the statistical perspective, neural networks
form an excellent hypothesis class; On one hand, for every runtime $T(d)$, by using depth of $O(T(d))$ we contain all predictors that can be run in time at most $T(d)$. On the other hand, the sample complexity of the resulting class depends polynomially on $T(d)$.

The main caveat of neural networks is the training time. Existing theoretical
results are mostly negative, showing that successfully learning with these
networks is computationally hard in the worst case. For example, neural
networks of depth $2$ contain the class of intersection of halfspaces (where
the number of halfspaces is the number of neurons in the hidden layer). By
reduction to $k$-coloring, it has been shown that finding the weights that
best fit the training set is NP-hard (\cite{blum1992training}).
\cite{BartlettB02} has shown that even finding weights that result in
close-to-minimal empirical error is computationally infeasible.  These
hardness results focus on \emph{proper} learning, where the goal is to find a
nearly-optimal predictor with a fixed network architecture $A$. However, if
our goal is to find a good predictor, there is no reason to limit ourselves
to predictors with one particular architecture. Instead, we can try, for
example, to find a network with a different architecture $A'$, which is
almost as good as the best network with architecture $A$. This is an example
of the powerful concept of \emph{improper} learning, which has often proved
useful in circumventing computational hardness results. Unfortunately, there
are hardness results showing that even with improper learning, and even if
the data is generated exactly from a small, depth-$2$ neural network, there
are no efficient algorithms which can find a predictor that performs well on
test data. In particular, \cite{KlivansSh06} and \cite{DanielyLiSh14stoc}
have shown this in the case of learning intersections of halfspaces, using
cryptographic and average case complexity assumptions. On a related note,
\cite{arora2013provable} recently showed positive results on learning from
data generated by a neural network of a certain architecture and randomly
connected weights. However, the assumptions used are strong and unlikely to
hold in practice.


Despite this theoretical pessimism, in practice, modern-day neural networks are trained successfully in many learning problems. There are several tricks that enable successful training:
\begin{itemize}[leftmargin=*]
\item \emph{Changing the activation function:} The threshold activation function, $\sigma(a) = \mathbf{1}_{a > 0}$, has zero derivative almost everywhere. Therefore, we cannot apply gradient-based methods with this activation function. To circumvent this problem, we can consider other activation functions. Most widely known is a sigmoidal activation, e.g. $\sigma(a) = \frac{1}{1+e^a}$, which forms a smooth approximation of the threshold function. Another recent popular activation function is the rectified linear unit (ReLU) function, $\sigma(a)=\max\{0,a\}$. Note that subtracting a shifted ReLU from a ReLU yields an approximation of the threshold function, so by doubling the number of neurons we can approximate a network with threshold activation by a network with ReLU activation.
\item \emph{Over-specification:} It was empirically observed that it is easier to train networks which are larger than needed. Indeed, we empirically demonstrate this phenomenon in \secref{sec:experiments}.
\item \emph{Regularization:} It was empirically observed that regularizing the weights of the network speeds up the convergence (e.g. \cite{krizhevsky2012imagenet}).
\end{itemize}

The goal of this paper is to revisit and re-raise the question of neural
network's computational efficiency, from a modern perspective. This is a
challenging topic, and we do not pretend to give
any definite answers. However, we provide several results, both
positive and negative. Most of them are new, although a few appeared
in the literature in other contexts. Our contributions are as
follows:

\begin{itemize}[leftmargin=*]
\item We make a simple observation that for sufficiently over-specified networks,
  global optima are ubiquitous and in general computationally easy to find. Although this
  holds only for extremely large networks which will overfit, it can be seen as an indication
  that the computational hardness of learning does decrease with the amount of over-specification. This is also demonstrated empirically in \secref{sec:experiments}. 
\item Motivated by the idea of changing the activation function, we
  consider the quadratic activation function, $\sigma(a)=a^2$. Networks with the quadratic activation compute polynomial functions of the input in $\reals^d$, hence we call them polynomial networks. Our main findings for such networks are as follows:
\begin{itemize}[leftmargin=*]
\item Networks with quadratic activation are as expressive as networks with threshold activation.

\item Constant depth networks with quadratic activation
  can be learned in polynomial time.

\item Sigmoidal networks of depth $2$, and with $\ell_1$
  regularization, can be approximated by polynomial networks of depth
  $O(\log\log(1/\epsilon))$. It follows
  that sigmoidal networks with $\ell_1$ regularization can be learned
  in polynomial time as well.

\item The aforementioned positive results are interesting
  theoretically, but lead to impractical algorithms. We provide a
  practical, provably correct, algorithm for training depth-$2$
  polynomial networks. While such networks can also be learned using
  a linearization trick, our algorithm is more efficient and returns networks whose size
  does not depend on the data dimension. Our algorithm follows a forward greedy selection procedure, where each step of the greedy selection procedure builds a new neuron by solving an eigenvalue problem.

\item We generalize the above algorithm to depth-$3$, in which each forward greedy step involves an efficient approximate solution to a tensor approximation problem. The algorithm can learn a rich sub-class of depth-$3$ polynomial networks.

\item We describe some experimental evidence, showing that our practical algorithm is competitive with state-of-the-art neural network training methods for depth-$2$ networks.
\end{itemize}
\end{itemize}

\section{Sufficiently Over-Specified Networks Are Easy to Train}\label{sec:overeasy}

We begin by considering the idea of over-specification, and make an observation that for
sufficiently over-specified networks, the optimization problem associated
with training them is generally quite easy to solve, and that global optima
are in a sense ubiquitous. As an interesting contrast, 
note that for very small networks (such as a single
neuron with a non-convex activation function), the associated optimization problem is generally
hard, and can exhibit exponentially many local (non-global) minima \cite{auer1996exponentially}.
We emphasize that our observation only holds for extremely large networks, which will
overfit in any reasonable scenario, but it does point to a possible spectrum where
computational cost decreases with the amount of over-specification.

To present the result, let $X\in \reals^{d,m}$ be a matrix of $m$ training examples in
$\reals^d$. We can think of the network as composed of two mappings. 
The first maps $X$ into a matrix $Z \in \reals^{n,m}$,
where $n$ is the number of neurons whose outputs are connected to
the output layer. The second mapping is a linear mapping $Z \mapsto
W Z$, where $W \in \reals^{o,n}$, that maps $Z$ to the $o$ neurons
in the output layer. Finally, there is a loss function $\ell :
\reals^{o,m} \to \reals$, which we'll assume to be convex, that assesses the quality of the prediction
on the entire data (and will of course depend on the $m$
labels). Let $V$ denote all the weights that affect the mapping from
$X$ to $Z$, and denote by $f(V)$ the function that maps $V$ to
$Z$. The optimization problem associated with learning the
network is therefore $\min_{W,V} \ell(W~f(V))$.

The function $\ell(W~f(V))$ is generally non-convex, and may have local minima.
However, if $n\geq m$, then it is reasonable to assume that $\text{Rank}(f(V))=m$
with large probability (under some random choice of $V$), due to the non-linear nature
of the function computed by neural networks\footnote{For example, consider the function computed by the first
layer, $X\mapsto \sigma(V_d X)$, where $\sigma$ is a sigmoid function. Since $\sigma$ is
non-linear, the columns of $\sigma(V_d X)$ will not be linearly
dependent in general.}. In that case, we can simply fix $V$ and solve
$\min_{W} \ell(W~f(V))$, which is computationally tractable as $\ell$ is assumed to be convex.
Since $f(V)$ has full rank, the solution of this problem corresponds to a global optima of $\ell$,
and hence to a global optima of the original optimization problem. Thus, for sufficiently large networks, 
finding global optima is generally easy, and they are in a sense ubiquitous.

%

\section{The Hardness of Learning Neural Networks} \label{sec:hardness}
We now review several known hardness results and apply them to our learning
setting. For simplicity, throughout most of this
section we focus on the PAC model in the binary classification case,
over the Boolean cube, in the realizable case, and with a fixed target
accuracy.\footnote{While we focus on the realizable case (i.e., there exists $f^* \in H$ that provides perfect predictions), with a fixed accuracy ($\epsilon$) and confidence ($\delta$), since we are dealing with hardness results, the results trivially apply to the agnostic case and to learning with arbitrarily small accuracy and confidence parameters.}

Fix some $\epsilon,\delta \in (0,1)$.
For every dimension $d$, let the input space be
$\mathcal{X}_d=\{0,1\}^d$ and let $H$ be a hypothesis class of
functions from $\X_d$ to $\{\pm 1\}$. We often omit the subscript $d$
when it is clear from context.  A learning algorithm $A$ has
access to an oracle that samples $\xx$ according to an unknown distribution $D$ over $\X$ and returns $(\xx,f^*(\xx))$, where $f^*$ is some unknown target hypothesis in $H$.
  The objective of the
algorithm is to return a classifier $f : \mathcal{X} \to \{\pm 1\}$,
such that with probability of at least $1-\delta$,
\[ \mathbb{P}_{\xx\sim D}\left[f(\xx) \neq f^*(\xx)\right]\le \epsilon.\]

We say that $A$ is efficient if it runs in time $\poly(d)$ and the
function it returns can also be evaluated on a new instance in time
$\poly(d)$. If there is such $A$, we say that $H$ is \emph{efficiently learnable}.

In the context of neural networks, every network architecture defines a
hypothesis class, $\nn_{t,n,\sigma}$, that contains all target
functions $f$ that can be implemented using a neural network with $t$
layers, $n$ neurons (excluding input neurons), and an activation
function $\sigma$. The immediate question is which
$\nn_{t,n,\sigma}$ are efficiently learnable. We will first address
this question for the threshold activation function, $\sigma_{0,1}(z)
= 1$ if $z>0$ and $0$ otherwise.

Observing that depth-$2$ networks with the threshold activation function can implement intersections of halfspaces, we will rely on the following hardness results, due to \cite{KlivansSh06}.

\begin{theorem}[Theorem 1.2 in \cite{KlivansSh06}] \label{thm:hard}
Let $\mathcal{X}=\{\pm 1\}^d$, let
\[H^a= \left\{ \xx \to \sigma_{0,1}\left(\ww^\top\xx -b- 1/2\right) ~:~ b \in \mathbb{N}, ~ \ww \in \mathbb{N}^d, |b|+\|\ww\|_1 \le \poly(d)\right\},\]
and let $H^a_k= \{\xx \to h_1(\xx)\wedge h_2(\xx)\wedge \ldots\wedge h_k(\xx) : \forall i, h_i \in H^a\}$, where $k = d^\rho$ for some constant $\rho > 0$. Then under a certain cryptographic assumption, $H^a_k$ is not efficiently learnable.
\end{theorem}
Under a different complexity assumption, \cite{DanielyLiSh14stoc}
showed a similar result even for $k=\omega(1)$.

As mentioned before, neural networks of depth $\ge 2$ and with the $\sigma_{0,1}$ activation function can express intersections of halfspaces: For example, the first layer consists of $k$ neurons computing the $k$ halfspaces, and the second layer computes their conjunction by the mapping $\xx\mapsto \sigma_{0,1}\left(\sum_{i}x_i-k+1/2\right)$. Trivially, if some class $H$ is not efficiently learnable, then any class containing it is also not efficiently learnable. We thus obtain the following corollary:
\begin{corollary}\label{cor:threshHard}
For every $t \ge 2, n = \omega(1)$, the class $\nn_{t,n,\sigma_{0,1}}$ is not efficiently learnable (under the complexity assumption given in \cite{DanielyLiSh14stoc}).
\end{corollary}

What happens when we change the activation function? In particular,
two widely used activation functions for neural networks are the
sigmoidal activation function, $\sigma_{\sig}(z) = 1/(1+\exp(-z))$,
and the rectified linear unit (ReLU) activation function,
$\sigma_{\relu}(z) = \max\{z,0\}$.

As a first observation, note that for $|z| \gg 1$ we have that
$\sigma_{\sig}(z) \approx \sigma_{0,1}(z)$. Our data domain
is the discrete Boolean cube, hence if we allow the
weights of the network to be arbitrarily large, then $\nn_{t,n,\sigma_{0,1}} \subseteq \nn_{t,n,\sigma_{\sig}}$. Similarly,
the function $\sigma_{\relu}(z) - \sigma_{\relu}(z-1)$ equals
$\sigma_{0,1}(z)$ for every $|z| \ge 1$. As a result, without
restricting the weights, we can simulate each threshold activated
neuron by two ReLU activated neurons, which implies that
$\nn_{t,n,\sigma_{0,1}} \subseteq \nn_{t,2n,\sigma_{\relu}}$.  Hence,
\corref{cor:threshHard} applies to both sigmoidal networks and ReLU
networks as well, as long as we do not regularize the weights of the
network.

What happens when we do regularize the weights? Let
$\nn_{t,n,\sigma,L}$ be all target functions that can be implemented
using a neural network of depth $t$, size $n$, activation function
$\sigma$, and when we restrict the input weights of each neuron to be
$\|\ww\|_1+|b| \le L$.

One may argue that in many real world distributions, the difference between the two classes, $\nn_{t,n,\sigma,L}$ and $\nn_{t,n,\sigma_{0,1}}$ is small. Roughly speaking, when the distribution density is low around the decision boundary of neurons (similarly to separation with margin assumptions), then sigmoidal neurons will be able to effectively simulate threshold activated neurons.

In practice, the sigmoid and ReLU activation functions are
advantageous over the threshold activation function, since they can be
trained using gradient based methods. Can these empirical successes be
turned into formal guarantees? Unfortunately, a closer examination of
\thmref{thm:hard} demonstrates that if $L = \Omega(d)$ then learning
$\nn_{2,n,\sigma_{\sig},L}$ and  $\nn_{2,n,\sigma_{\relu},L}$ is still
hard. Formally, to apply these networks to
binary classification, we follow a standard definition of learning with
a margin assumption: We assume that the learner receives examples of the form
$(\xx,\sign(f^*(\xx)))$ where $f^*$ is a real-valued function that
comes from the hypothesis class, and we further assume that
$|f^*(\xx)| \ge 1$. Even under this margin assumption, we have the following:
\begin{corollary}\label{cor:sigmoidHard}
For every $t \ge 2, n = \omega(1)$, $L = \Omega(d)$, the classes $\nn_{t,n,\sigma_{\sig},L}$ and $\nn_{t,n,\sigma_{\relu},L}$ are not efficiently learnable (under the complexity assumption given in \cite{DanielyLiSh14stoc}).
\end{corollary}
A proof is provided in the appendix. What happens when $L$ is much smaller? Later on in the paper we will show positive results for $L$ being a constant and the depth being fixed. These results will be obtained using polynomial networks, which we study in the next section.

\section{Polynomial Networks}

In the previous section we have shown several strong negative results
for learning neural networks with the threshold, sigmoidal, and ReLU
activation functions. One way to circumvent these hardness results is
by considering another activation function. Maybe the simplest
non-linear function is the squared function, $\sigma_{2}(x) = x^2$. We
call networks that use this activation function \emph{polynomial
  networks}, since they compute polynomial functions of their
inputs. As in the previous section, we denote by
$\nn_{t,n,\sigma_2,L}$ the class of functions that can be implemented
using a neural network of depth $t$, size $n$, squared activation
function, and a bound $L$ on the $\ell_1$ norm of the input weights of
each neuron. Whenever we do not specify $L$ we refer to polynomial
networks with unbounded weights.

Below we study the expressiveness and computational complexity of polynomial
networks. We note that algorithms for efficiently learning (real-valued)
sparse or low-degree polynomials has been studied in several previous works
(e.g.
\cite{KaKlMaSe08,KaSaTe09,BlOdWi10,andoni2013learning,andoni2014learning}).
However, these rely on strong distributional assumptions, such as the data
instances having a uniform or log-concave distribution, while we are
interested in a distribution-free setting.

\subsection{Expressiveness}

We first show that, similarly to networks with threshold
activation, polynomial networks of polynomial size can
express all functions that can be implemented efficiently using a
Turing machine.
\begin{theorem}[Polynomial networks can express Turing Machines]\label{thm:turingPoly}
Let $\mathcal{F}_d$ and $T$ be as in \thmref{thm:turing}.
 Then there exist constants $b,c \in
  \reals_+$ such that for every $d$, the class $\nn_{t,n,\sigma_2,L}$,
  with $t=c\,T(d)\log(T(d))+b$, $n=t^2$, and $L=b$, contains
  $\mathcal{F}_d$.
\end{theorem}
The proof of the theorem relies on the result of
\cite{pippenger1979relations} and is given in the appendix.

Another relevant expressiveness result, which we will use later,
shows that polynomial networks can approximate networks with
sigmoidal activation functions:
\begin{theorem}\label{thm:ntop}
Fix $0<\epsilon<1,~ L\ge 3$ and $t\in\mathbb{N}$. There are $B_t\in \tilde{O}(\log(tL+L\log\frac{1}{\epsilon}))$ and $B_n \in \tilde{O}(tL+L\log\frac{1}{\epsilon})$ such that
for every $f\in \nn_{t,n,\sigma_{\sig},L}$ there is a function $g\in \nn_{tB_t,nB_n,\sigma_2}$,  such that
$ \sup_{\|\xx\|_{\infty}<1} \|f(\xx)- g(\xx) \|_{\infty}\le\epsilon $.
\end{theorem}
The proof relies on an approximation of the sigmoid function based on
Chebyshev polynomials, as was done in \cite{Shamir2010Learning}, and is given in the appendix.

\ignore{
Last, the following lemma provides some building blocks for the
construction of polynomials using polynomial networks.
\begin{lemma}\label{thm:polyexpressive} The following target functions can be implemented via a polynomial network
\begin{enumerate}
\item\label{id} The target function $I(x)=x$ is in $\nn_{2,2,\sigma_2}$.
\item\label{prod} Every target function of the form $p(\xx)=(\ww_1^\top
  \xx)\cdot (\ww_2^\top \xx)$, with  $\|\ww_1\|_1,\|\ww_2\|_1 \le L$ for some $L\ge 2$, is in $\nn_{2,3,\sigma_2,L}$.
\item\label{power} Every target function of the form $m(\xx)=(\ww^\top
  \xx)^\top$, with $\|\ww\|_1 \le L$,  is in $\nn_{t,n,\sigma_2,N}$, where $t=\log T+ \log\log T$ and $n= 2\log T + \log T( \log\log T)$.
\item\label{poly} Every target function of the form $q(\ww^\top \xx) =
  \sum_{k=1}^T a_i (\ww^\top\xx)^k$, with $\|a\|_1,\|\ww\|_1 \le L$ for some
  $L\ge 2$,  is in $\nn_{t,n,\sigma_2,L}$ where $t=\log T+\log\log T$ and $n=2\|a\|_0(2\log T+\log T(\log\log T)$, where $\|a\|_0=|\{k:a_k\ne 0\}|$.
\end{enumerate}
\end{lemma}
}

\subsection{Training Time}\label{subsec:trainingtime}

We now turn to the computational complexity of learning polynomial
networks. We first show that it is hard to learn polynomial networks
of depth $\Omega(\log(d))$. Indeed, by combining \thmref{thm:ntop} and \corref{cor:sigmoidHard} we obtain
the following:
\begin{corollary}
The class $\nn_{t,n,\sigma_2}$, where $t = \Omega(\log(d))$ and $n=\Omega(d)$, is not
efficiently learnable.
\end{corollary}

On the flip side, constant-depth polynomial networks can be learned in polynomial
time, using a simple linearization trick. Specifically,
the class of polynomial networks of constant depth $t$ is
contained in the class of multivariate polynomials of total degree at
most $s=2^t$. This class can be represented as a $d^s$-dimensional linear
space, where each vector is the coefficient vector of some such polynomial.
Therefore, the class of polynomial networks of depth $t$ can be learned in time
$\poly(d^{2^t})$, by mapping each
instance vector $\xx \in \reals^d$ to all of its monomials, and learning
a linear predictor on top of this representation (which can be done efficiently
in the realizable case, or when a convex loss function is used).
In particular, if $t$ is a constant then so is $2^t$ and therefore
polynomial networks of constant depth are efficiently learnable.
Another way to learn this class is using support vector machines with polynomial kernels.

An interesting application of this observation is that depth-$2$ sigmoidal
networks are efficiently learnable with sufficient regularization, as formalized
in the result below. This contrasts with corollary \ref{cor:sigmoidHard}, which
provides a hardness result without regularization.
\begin{theorem} \label{thm:learnable}
The class $\nn_{2,n,\sigma_{\sig},L}$ can be learned, to accuracy $\epsilon$, in time $\poly(T)$ where
$T=(1/\epsilon) \cdot O(d^{4L\ln(11L^2+1)})$.
\end{theorem}
The idea of the proof is as follows. Suppose that we obtain data from
some $f \in \nn_{2,n,\sigma_{\sig},L}$. Based on \thmref{thm:ntop},
there is $g\in \nn_{2B_t ,nB_n,\sigma_2}$ that approximates $f$ to
some fixed accuracy $\epsilon_0 = 0.5$, where $B_t$ and $B_n$ are as
defined in \thmref{thm:ntop} for $t=2$. Now we can learn
$\nn_{2B_t,nB_n,\sigma_2}$ by considering the class of all polynomials
of total degree $2^{2B_t}$, and applying the linearization technique
discussed above. Since $f$ is assumed to separate the data with
margin $1$ (i.e. $y=\sign(f^*(\xx))$,$|f^*(\xx)|\geq 1|$),
then $g$ separates the data with margin $0.5$, which is
enough for establishing accuracy $\epsilon$ in sample and time that
depends polynomially on $1/\epsilon$.

\subsection{Learning 2-layer and 3-layer Polynomial Networks}

While interesting theoretically, the above results are not very practical,
since the time and sample complexity grow very fast with the depth of
the network.\footnote{If one uses SVM with polynomial kernels,
  the time and sample complexity may be small under margin assumptions
  in a feature space corresponding to a given kernel. Note, however, that large margin in that space
  is very different than the assumption we make here, namely,
that there is a network with a small number of hidden neurons that works
well on the data.}  In this section we describe practical,
provably correct, algorithms for the special case of depth-$2$ and
depth-$3$ polynomial networks, with some additional constraints.
Although such networks can be learned in polynomial time via explicit
linearization (as described in
section \ref{subsec:trainingtime}), the runtime and resulting network size
scales quadratically (for depth-2) or cubically (for depth-3) with the data dimension $d$.
In contrast, our algorithms and guarantees have a much milder dependence on $d$.

We first consider 2 layer polynomial networks, of the following form:
\[
\dpn_{2,k} = \left\{ \xx \mapsto b + \ww_0^\top \xx +  \sum_{i=1}^k \alpha_i (\ww_i^\top \xx)^2 :
~~\forall i \ge 1, |\alpha_i| \le 1, \|\ww_i\|_2 = 1 \right\} ~.
\]
This networks corresponds to one hidden layer containing $r$ neurons
with the squared activation function, where we restrict the input
weights of all neurons in the network to have bounded $\ell_2$ norm, and where
we also allow a direct linear dependency between the input layer and
the output layer.

We'll describe an efficient algorithm for learning this class, which
is based on the GECO algorithm for convex optimization
with low-rank constraints \cite{shalev2011large}.

The goal of the algorithm is to find $f$ that minimizes the objective
\begin{equation}\label{eq:Rloss}
R(f)= \frac{1}{m}\sum_{i=1}^m  \ell(f(\xx_i),y_i),
\end{equation}
where $\ell : \reals \times \reals
\to \reals$ is a loss function. We'll assume that $\ell$ is
$\beta$-smooth and convex.

The basic idea of the algorithm is to gradually add hidden neurons to
the hidden layer, in a greedy manner, so as to decrease the loss function
over the data. To do so, define $\mathcal{V} = \{\xx \mapsto
(\ww^\top\xx)^2 : \|\ww\|_2=1\}$ the set of functions that can be
implemented by hidden neurons. Then every $f \in \dpn_{2,r}$ is an
affine function plus a weighted sum of functions from
$\mathcal{V}$.  The algorithm starts with $f$ being the minimizer of
$R$ over all affine functions. Then at each greedy step, we search
for $g \in \mathcal{V}$ that minimizes a first order approximation of
$R(f + \eta g)$:
\begin{equation} \label{eqn:Rapprox}
R(f + \eta g) \approx R(f) + \eta\, \frac{1}{m} \sum_{i=1}^m
\ell'(f(\xx_i),y_i) g(\xx_i) ~,
\end{equation}
where $\ell'$ is the derivative of $\ell$ w.r.t. its first argument.
Observe that for every $g \in \mathcal{V}$ there is some $\ww$ with
$\|\ww\|_2=1$ for which $g(\xx) = (\ww^\top \xx)^2 = \ww^\top \xx
\xx^\top \ww$. Hence, the right-hand side of \eqref{eqn:Rapprox} can
be rewritten as
$
 R(f) + \eta~ \ww^\top \left( \frac{1}{m} \sum_{i=1}^m
\ell'(f(\xx_i),y_i) \xx_i \xx_i^\top \right) \ww ~.
$
The vector $\ww$ that minimizes this expression (for positive $\eta$) is the leading eigenvector of the
matrix $\left( \frac{1}{m} \sum_{i=1}^m \ell'(f(\xx_i),y_i) \xx_i
  \xx_i^\top \right)$. We add this vector as a hidden neuron to the
network.\footnote{It is also possible to find an approximate solution
  to the eigenvalue problem and still retain the performance
  guarantees (see \cite{shalev2011large}). Since an approximate
  eigenvalue can be found in time $O(d)$ using the power method, we
  obtain the runtime of GECO depends linearly on $d$.} Finally, we
minimize $R$ w.r.t. the weights from the hidden layer to the output
layer (namely, w.r.t. the weights $\alpha_i$).

The following theorem, which follows directly from Theorem 1 of
\cite{shalev2011large}, provides convergence guarantee for
GECO. Observe that the theorem gives guarantee for learning
$\dpn_{2,k}$ if we allow to output an over-specified network.
\begin{theorem} \label{thm:Geco}
Fix some $\epsilon>0$. Assume that the loss function is convex and
$\beta$-smooth. Then if the GECO Algorithm is run for $r >
\frac{2\beta k^2}{\epsilon}$ iterations, it outputs a network $f \in \nn_{2,r,\sigma_2}$ for
which $R(f) \le \min_{f^* \in \dpn_{2,k}} R(f^*) + \epsilon$.
\end{theorem}

We next consider a hypothesis class consisting of third degree
polynomials, which is a subset of $3$-layer polynomial networks (see Lemma $1$ in the appendix) . The
hidden neurons will be functions from the class:
$
\mathcal{V}=\cup_{i=1}^{3} \mathcal{V}_i ~~~\textrm{where}~~~
\mathcal{V}_i = \left\{\xx \mapsto \prod_{j=1}^i (\ww_j^\top\xx) :
  \forall j, ~\|\ww_j\|_2 = 1\right\} ~.
$
The hypothesis class we consider is
$
\dpn_{3,k} ~=~ \left\{ \xx \mapsto \sum_{i=1}^k
\alpha_i g_i(\xx) : \forall i,~ |\alpha_i| \le 1, g_i \in
\mathcal{V} \right \}.
$

The basic idea of the algorithm is the same as for 2-layer
networks. However, while in the 2-layer case we could implement
efficiently each greedy step by solving an eigenvalue problem, we now
face the following tensor approximation problem at each greedy step:
\[
\max_{g \in \mathcal{V}_3}  \frac{1}{m} \sum_{i=1}^m
\ell'(f(\xx_i),y_i) g(\xx_i) = \max_{\|\ww\|=1,\|\uu\|=1,\|\vv\|=1} \frac{1}{m} \sum_{i=1}^m
\ell'(f(\xx_i),y_i) (\ww^\top \xx_i) (\uu^\top \xx_i) (\vv^\top \xx_i) ~.
\]
While this is in general a hard optimization problem, we can
approximate it -- and luckily, an approximate greedy step suffices for
success of the greedy procedure. This procedure is given
in Figure~1, and is again based on an approximate eigenvector
computation. A guarantee for the quality of approximation
is given in the appendix, and this leads to the following theorem,
whose proof is given in the appendix.
\begin{theorem} \label{thm:3Gec}
Fix some $\delta,\epsilon>0$. Assume that the loss function is convex and
$\beta$-smooth. Then if the GECO Algorithm is run for
$r>\frac{4d\beta k^2}{\epsilon(1-\tau)^2}$ iterations, where each
iteration relies on the approximation procedure given in
\figref{fig:tensor}, then with probability $(1-\delta)^r$,  it outputs a network $f \in \nn_{3,5r,\sigma_2}$ for
which $R(f) \le \min_{f^* \in \dpn_{3,k}} R(f^*) + \epsilon$.
\end{theorem}

\begin{figure}
\begin{tikzpicture}

\node at (0,0) {
\fbox{
\small
\begin{Balgorithm} 
\textbf{Input}: $\{x_i\}_{i=1}^m\in \mathbb{R}^d$ $\alpha\in \mathbb{R}^m$, $\tau$,$\delta$ \\
\textbf{Output}: A $\frac{1-\tau}{\sqrt{d}}$ approximate solution to
\\$\displaystyle \max_{\|\ww\|,\|\uu\|,\|\vv\|=1} F(\ww,\uu,\vv)=\sum_i \alpha_i (\ww^\top \xx_i) (\uu^\top\xx_i)(\vv^\top \xx_i)$\\
Pick randomly $\ww_1,\ldots,\ww_s$ iid according to $\mathcal{N}(0,I_d)$.\\
  \textbf{For} $t=1,\ldots ,2d\log\frac{1}{\delta}$\+\+\\
  $\ww_t \leftarrow \frac{\ww_t}{\|\ww_t\|}$\\
  Let $A=\sum_i \alpha_i (\ww_t^\top \xx_i) \xx_i\xx_i^\top$ and set
  $\uu_t,\vv_t$ s.t:\\
 $Tr(\uu_t^\top A\vv_t) \ge (1-\tau)   \max_{\|\uu\|,\|\vv\|=1} Tr(\uu^\top A\vv)$. \-\-\\
Return $\ww,\uu, \vv$ the maximizers of $\max_{i\le s} F(\ww_i,\uu_i,\uu_i)$.
\end{Balgorithm}}
};
\node at (0,-2.5) {Figure 1: Approximate tensor maximization.};

\node at (7,0) {\includegraphics[totalheight=0.2\textheight]{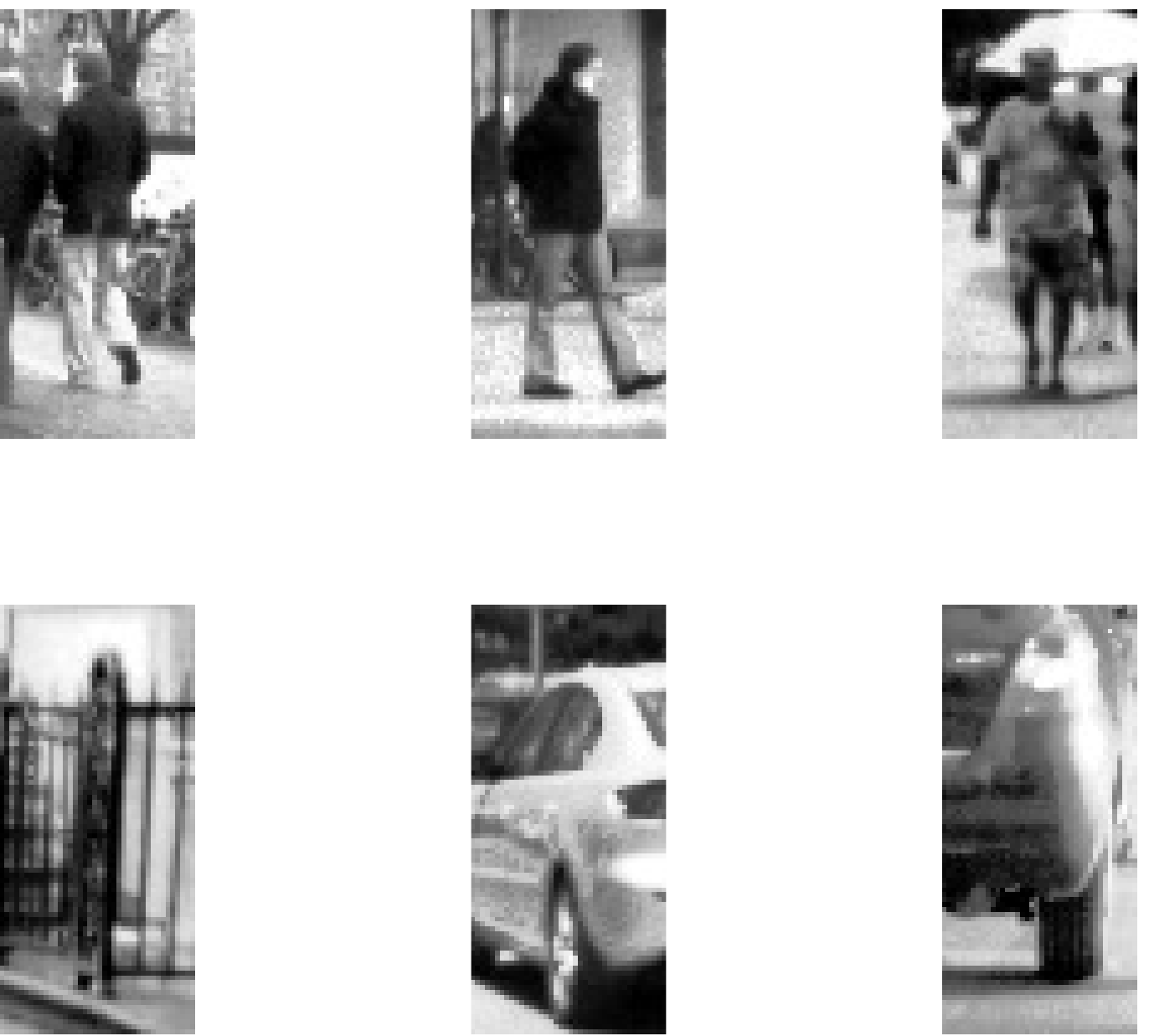}};

\end{tikzpicture}
\refstepcounter{figure}\label{fig:tensor} \vskip -0.5cm
\end{figure}

\vspace{-0.4cm}
\section{Experiments} \label{sec:experiments}
To demonstrate the practicality of GECO to train neural networks for real world problems, we
considered a pedestrian detection problem as follows. We
collected 200k training examples of image patches of size 88x40 pixels
containing either pedestrians (positive examples) or hard negative examples (containing
images that were classified as pedestrians by
applying a simple linear classifier in a sliding window manner).
See a few examples of images above. We used half of the examples as a
\begin{wrapfigure}{r}{0.34\textwidth}
  \begin{center}
\hfill
\begin{tikzpicture}[scale=0.53]
    \begin{axis}[
     xlabel={iterations},
     ylabel={Error}
    ]

\addplot[blue,line width=2pt] coordinates {
(1000,0.0964) (2000,0.0982) (3000,0.0857) (4000,0.0804) (5000,0.0763) (6000,0.0729) (7000,0.0726) (8000,0.0695) (9000,0.0691) (10000,0.0698) (11000,0.0705) (12000,0.0639) (13000,0.0627) (14000,0.0645) (15000,0.0614) (16000,0.0638) (17000,0.0609) (18000,0.0604) (19000,0.0635) (20000,0.0607) (21000,0.0613) (22000,0.0596) (23000,0.0604) (24000,0.0605) (25000,0.0589) (26000,0.0598) (27000,0.0598) (28000,0.0626) (29000,0.0587) (30000,0.0594) (31000,0.0584) (32000,0.0587) (33000,0.0587) (34000,0.0595) (35000,0.0588) (36000,0.0577) (37000,0.0578) (38000,0.0586) (39000,0.0564) (40000,0.0593) (41000,0.0568) (42000,0.0569) (43000,0.0583) (44000,0.0581) (45000,0.0581) (46000,0.0585) (47000,0.0581) (48000,0.0561) (49000,0.0568) (50000,0.0573) (51000,0.0566) (52000,0.0564) (53000,0.0562) (54000,0.0567) (55000,0.0570) (56000,0.0569) (57000,0.0570) (58000,0.0562) (59000,0.0564) (60000,0.0560) (61000,0.0573) (62000,0.0563) (63000,0.0569) (64000,0.0578) (65000,0.0566) (66000,0.0567) (67000,0.0569) (68000,0.0569) (69000,0.0573) (70000,0.0557) (71000,0.0558) (72000,0.0565) (73000,0.0557) (74000,0.0566) (75000,0.0545) (76000,0.0554) (77000,0.0556) (78000,0.0561) (79000,0.0555) (80000,0.0552) (81000,0.0545) (82000,0.0551) (83000,0.0551) (84000,0.0559) (85000,0.0551) (86000,0.0547) (87000,0.0552) (88000,0.0567) (89000,0.0549) (90000,0.0553) (91000,0.0555) (92000,0.0548) (93000,0.0555) (94000,0.0548) (95000,0.0555) (96000,0.0549) (97000,0.0554) (98000,0.0559) (99000,0.0546) (100000,0.0549)
};
\addlegendentry{SGD ReLU};

\addplot[red,dashed,line width=2pt] coordinates {
(1000,0.0942) (2000,0.0853) (3000,0.0861) (4000,0.0760) (5000,0.0743) (6000,0.0726) (7000,0.0691) (8000,0.0682) (9000,0.0646) (10000,0.0652) (11000,0.0628) (12000,0.0634) (13000,0.0611) (14000,0.0595) (15000,0.0602) (16000,0.0579) (17000,0.0639) (18000,0.0635) (19000,0.0593) (20000,0.0626) (21000,0.0568) (22000,0.0580) (23000,0.0588) (24000,0.0561) (25000,0.0570) (26000,0.0552) (27000,0.0551) (28000,0.0555) (29000,0.0549) (30000,0.0586) (31000,0.0546) (32000,0.0561) (33000,0.0548) (34000,0.0599) (35000,0.0544) (36000,0.0547) (37000,0.0568) (38000,0.0551) (39000,0.0544) (40000,0.0540) (41000,0.0536) (42000,0.0537) (43000,0.0539) (44000,0.0530) (45000,0.0532) (46000,0.0543) (47000,0.0536) (48000,0.0541) (49000,0.0538) (50000,0.0548) (51000,0.0534) (52000,0.0541) (53000,0.0547) (54000,0.0541) (55000,0.0552) (56000,0.0539) (57000,0.0531) (58000,0.0549) (59000,0.0555) (60000,0.0584) (61000,0.0538) (62000,0.0523) (63000,0.0530) (64000,0.0540) (65000,0.0542) (66000,0.0534) (67000,0.0529) (68000,0.0555) (69000,0.0535) (70000,0.0531) (71000,0.0535) (72000,0.0530) (73000,0.0529) (74000,0.0525) (75000,0.0533) (76000,0.0551) (77000,0.0533) (78000,0.0528) (79000,0.0530) (80000,0.0538) (81000,0.0533) (82000,0.0536) (83000,0.0542) (84000,0.0518) (85000,0.0531) (86000,0.0526) (87000,0.0525) (88000,0.0530) (89000,0.0540) (90000,0.0554) (91000,0.0524) (92000,0.0531) (93000,0.0528) (94000,0.0530) (95000,0.0519) (96000,0.0530) (97000,0.0529) (98000,0.0542) (99000,0.0528) (100000,0.0521)
};
\addlegendentry{SGD Squared};

\addplot[black,line width=2pt] coordinates {
(1000,0.05) (100000,0.05)
};
\addlegendentry{GECO};

\end{axis}
\end{tikzpicture}

\vspace{0.2cm}

\hfill
\begin{tikzpicture}[scale=0.53]
    \begin{axis}[
     xlabel={\#iterations},
     ylabel={MSE}
    ]

\addplot[blue,line width=2pt] coordinates {
(500,3.911) (1000,3.493) (1500,3.322) (2000,3.142) (2500,3.062) (3000,2.979) (3500,2.853) (4000,2.813) (4500,2.730) (5000,2.753) (5500,2.633) (6000,2.616) (6500,2.572) (7000,2.509) (7500,2.495) (8000,2.557) (8500,2.430) (9000,2.396) (9500,2.435) (10000,2.371) (10500,2.415) (11000,2.341) (11500,2.330) (12000,2.327) (12500,2.280) (13000,2.309) (13500,2.316) (14000,2.256) (14500,2.264) (15000,2.213) (15500,2.187) (16000,2.205) (16500,2.199) (17000,2.172) (17500,2.211) (18000,2.152) (18500,2.162) (19000,2.144) (19500,2.134) (20000,2.156) (20500,2.144) (21000,2.149) (21500,2.120) (22000,2.097) (22500,2.136) (23000,2.112) (23500,2.112) (24000,2.107) (24500,2.069) (25000,2.084) (25500,2.080) (26000,2.072) (26500,2.060) (27000,2.069) (27500,2.076) (28000,2.035) (28500,2.047) (29000,2.037) (29500,2.018) (30000,2.038) (30500,2.042) (31000,2.028) (31500,2.055) (32000,1.994) (32500,2.008) (33000,1.997) (33500,1.988) (34000,2.017) (34500,1.986) (35000,2.012) (35500,1.992) (36000,1.980) (36500,1.990) (37000,1.967) (37500,1.993) (38000,1.992) (38500,1.983) (39000,1.954) (39500,2.011) (40000,1.984) (40500,1.963) (41000,1.972) (41500,1.952) (42000,1.948) (42500,1.931) (43000,1.930) (43500,1.949) (44000,1.944) (44500,1.950) (45000,1.944) (45500,1.923) (46000,1.908) (46500,1.908) (47000,1.929) (47500,1.987) (48000,1.961) (48500,1.937) (49000,1.924) (49500,1.915) (50000,1.920) (50500,1.930) (51000,1.920) (51500,1.921) (52000,1.902) (52500,1.936) (53000,1.912) (53500,1.894) (54000,1.894) (54500,1.896) (55000,1.886) (55500,1.911) (56000,1.893) (56500,1.882) (57000,1.894) (57500,1.893) (58000,1.877) (58500,1.896) (59000,1.882) (59500,1.889) (60000,1.895) (60500,1.882) (61000,1.867) (61500,1.882) (62000,1.875) (62500,1.885) (63000,1.878) (63500,1.883) (64000,1.889) (64500,1.884) (65000,1.892) (65500,1.859) (66000,1.905) (66500,1.877) (67000,1.863) (67500,1.859) (68000,1.855) (68500,1.870) (69000,1.858) (69500,1.867) (70000,1.850) (70500,1.848) (71000,1.852) (71500,1.852) (72000,1.851) (72500,1.859) (73000,1.837) (73500,1.848) (74000,1.879) (74500,1.852) (75000,1.856) (75500,1.865) (76000,1.836) (76500,1.836) (77000,1.831) (77500,1.836) (78000,1.841) (78500,1.830) (79000,1.834) (79500,1.835) (80000,1.826) (80500,1.827) (81000,1.821) (81500,1.829) (82000,1.842) (82500,1.826) (83000,1.829) (83500,1.823) (84000,1.822) (84500,1.820) (85000,1.833) (85500,1.826) (86000,1.818) (86500,1.810) (87000,1.818) (87500,1.823) (88000,1.816) (88500,1.818) (89000,1.807) (89500,1.809) (90000,1.837) (90500,1.807) (91000,1.808) (91500,1.813) (92000,1.811) (92500,1.812) (93000,1.804) (93500,1.820) (94000,1.807) (94500,1.804) (95000,1.799) (95500,1.810) (96000,1.804) (96500,1.791) (97000,1.802) (97500,1.828) (98000,1.828) (98500,1.808) (99000,1.810) (99500,1.798) (100000,1.836)
};
\addlegendentry{1};

\addplot[red,line width=2pt] coordinates {
(500,3.558) (1000,3.125) (1500,2.945) (2000,2.754) (2500,2.636) (3000,2.561) (3500,2.387) (4000,2.337) (4500,2.254) (5000,2.156) (5500,2.093) (6000,2.030) (6500,1.941) (7000,1.978) (7500,1.887) (8000,1.891) (8500,1.802) (9000,1.787) (9500,1.743) (10000,1.712) (10500,1.660) (11000,1.655) (11500,1.641) (12000,1.607) (12500,1.575) (13000,1.598) (13500,1.568) (14000,1.504) (14500,1.496) (15000,1.480) (15500,1.449) (16000,1.471) (16500,1.440) (17000,1.423) (17500,1.384) (18000,1.376) (18500,1.377) (19000,1.361) (19500,1.382) (20000,1.310) (20500,1.317) (21000,1.316) (21500,1.261) (22000,1.248) (22500,1.248) (23000,1.251) (23500,1.233) (24000,1.212) (24500,1.217) (25000,1.204) (25500,1.205) (26000,1.200) (26500,1.189) (27000,1.191) (27500,1.163) (28000,1.152) (28500,1.165) (29000,1.141) (29500,1.127) (30000,1.129) (30500,1.116) (31000,1.083) (31500,1.107) (32000,1.087) (32500,1.075) (33000,1.070) (33500,1.068) (34000,1.057) (34500,1.056) (35000,1.047) (35500,1.055) (36000,1.033) (36500,1.034) (37000,1.019) (37500,1.024) (38000,1.015) (38500,1.002) (39000,1.001) (39500,0.992) (40000,1.016) (40500,1.007) (41000,0.977) (41500,0.976) (42000,0.980) (42500,0.955) (43000,0.957) (43500,0.974) (44000,0.951) (44500,0.953) (45000,0.946) (45500,0.942) (46000,0.939) (46500,0.949) (47000,0.926) (47500,0.931) (48000,0.915) (48500,0.919) (49000,0.907) (49500,0.927) (50000,0.916) (50500,0.900) (51000,0.892) (51500,0.887) (52000,0.901) (52500,0.896) (53000,0.879) (53500,0.886) (54000,0.893) (54500,0.870) (55000,0.870) (55500,0.866) (56000,0.863) (56500,0.855) (57000,0.861) (57500,0.857) (58000,0.856) (58500,0.849) (59000,0.856) (59500,0.844) (60000,0.840) (60500,0.831) (61000,0.839) (61500,0.825) (62000,0.829) (62500,0.823) (63000,0.824) (63500,0.829) (64000,0.810) (64500,0.814) (65000,0.805) (65500,0.808) (66000,0.798) (66500,0.804) (67000,0.797) (67500,0.805) (68000,0.799) (68500,0.792) (69000,0.786) (69500,0.782) (70000,0.786) (70500,0.777) (71000,0.787) (71500,0.777) (72000,0.775) (72500,0.768) (73000,0.762) (73500,0.775) (74000,0.773) (74500,0.765) (75000,0.774) (75500,0.769) (76000,0.762) (76500,0.760) (77000,0.747) (77500,0.749) (78000,0.760) (78500,0.738) (79000,0.749) (79500,0.741) (80000,0.740) (80500,0.733) (81000,0.734) (81500,0.735) (82000,0.724) (82500,0.731) (83000,0.736) (83500,0.732) (84000,0.730) (84500,0.721) (85000,0.714) (85500,0.715) (86000,0.711) (86500,0.713) (87000,0.703) (87500,0.709) (88000,0.698) (88500,0.694) (89000,0.702) (89500,0.705) (90000,0.700) (90500,0.691) (91000,0.687) (91500,0.686) (92000,0.687) (92500,0.679) (93000,0.686) (93500,0.689) (94000,0.681) (94500,0.673) (95000,0.675) (95500,0.676) (96000,0.668) (96500,0.675) (97000,0.670) (97500,0.663) (98000,0.665) (98500,0.664) (99000,0.666) (99500,0.657) (100000,0.662)
};
\addlegendentry{2};

\addplot[black,line width=2pt] coordinates {
(500,3.467) (1000,2.981) (1500,2.553) (2000,2.305) (2500,2.068) (3000,1.912) (3500,1.721) (4000,1.585) (4500,1.465) (5000,1.340) (5500,1.200) (6000,1.168) (6500,1.087) (7000,1.046) (7500,0.967) (8000,0.863) (8500,0.817) (9000,0.763) (9500,0.702) (10000,0.682) (10500,0.636) (11000,0.583) (11500,0.552) (12000,0.551) (12500,0.495) (13000,0.476) (13500,0.439) (14000,0.423) (14500,0.414) (15000,0.365) (15500,0.344) (16000,0.344) (16500,0.329) (17000,0.315) (17500,0.282) (18000,0.276) (18500,0.268) (19000,0.267) (19500,0.248) (20000,0.234) (20500,0.233) (21000,0.208) (21500,0.215) (22000,0.197) (22500,0.183) (23000,0.183) (23500,0.177) (24000,0.180) (24500,0.157) (25000,0.161) (25500,0.153) (26000,0.148) (26500,0.142) (27000,0.133) (27500,0.127) (28000,0.130) (28500,0.119) (29000,0.122) (29500,0.117) (30000,0.114) (30500,0.109) (31000,0.103) (31500,0.100) (32000,0.099) (32500,0.096) (33000,0.095) (33500,0.093) (34000,0.087) (34500,0.086) (35000,0.086) (35500,0.083) (36000,0.082) (36500,0.080) (37000,0.079) (37500,0.072) (38000,0.076) (38500,0.071) (39000,0.070) (39500,0.066) (40000,0.071) (40500,0.065) (41000,0.063) (41500,0.067) (42000,0.063) (42500,0.061) (43000,0.062) (43500,0.060) (44000,0.060) (44500,0.059) (45000,0.056) (45500,0.056) (46000,0.056) (46500,0.055) (47000,0.053) (47500,0.053) (48000,0.053) (48500,0.051) (49000,0.051) (49500,0.051) (50000,0.049) (50500,0.051) (51000,0.052) (51500,0.048) (52000,0.048) (52500,0.047) (53000,0.048) (53500,0.046) (54000,0.047) (54500,0.045) (55000,0.046) (55500,0.044) (56000,0.044) (56500,0.044) (57000,0.043) (57500,0.045) (58000,0.044) (58500,0.044) (59000,0.043) (59500,0.042) (60000,0.042) (60500,0.041) (61000,0.042) (61500,0.041) (62000,0.039) (62500,0.039) (63000,0.038) (63500,0.039) (64000,0.038) (64500,0.038) (65000,0.035) (65500,0.036) (66000,0.041) (66500,0.036) (67000,0.035) (67500,0.035) (68000,0.036) (68500,0.035) (69000,0.034) (69500,0.034) (70000,0.034) (70500,0.033) (71000,0.032) (71500,0.031) (72000,0.033) (72500,0.032) (73000,0.032) (73500,0.032) (74000,0.032) (74500,0.032) (75000,0.032) (75500,0.033) (76000,0.031) (76500,0.032) (77000,0.031) (77500,0.031) (78000,0.030) (78500,0.030) (79000,0.031) (79500,0.033) (80000,0.030) (80500,0.029) (81000,0.030) (81500,0.031) (82000,0.029) (82500,0.029) (83000,0.030) (83500,0.030) (84000,0.028) (84500,0.028) (85000,0.028) (85500,0.028) (86000,0.029) (86500,0.028) (87000,0.028) (87500,0.028) (88000,0.027) (88500,0.027) (89000,0.027) (89500,0.027) (90000,0.026) (90500,0.027) (91000,0.026) (91500,0.026) (92000,0.027) (92500,0.026) (93000,0.026) (93500,0.025) (94000,0.026) (94500,0.025) (95000,0.026) (95500,0.025) (96000,0.025) (96500,0.025) (97000,0.026) (97500,0.025) (98000,0.026) (98500,0.024) (99000,0.023) (99500,0.025) (100000,0.025)
};
\addlegendentry{4};

\addplot[green,line width=2pt] coordinates {
(500,3.326) (1000,2.673) (1500,2.345) (2000,2.022) (2500,1.601) (3000,1.408) (3500,1.248) (4000,1.024) (4500,0.852) (5000,0.734) (5500,0.646) (6000,0.537) (6500,0.461) (7000,0.364) (7500,0.310) (8000,0.267) (8500,0.224) (9000,0.182) (9500,0.152) (10000,0.123) (10500,0.107) (11000,0.091) (11500,0.084) (12000,0.065) (12500,0.061) (13000,0.052) (13500,0.039) (14000,0.037) (14500,0.031) (15000,0.031) (15500,0.027) (16000,0.023) (16500,0.022) (17000,0.021) (17500,0.018) (18000,0.018) (18500,0.017) (19000,0.016) (19500,0.015) (20000,0.015) (20500,0.014) (21000,0.015) (21500,0.015) (22000,0.014) (22500,0.013) (23000,0.014) (23500,0.014) (24000,0.012) (24500,0.013) (25000,0.013) (25500,0.012) (26000,0.012) (26500,0.013) (27000,0.012) (27500,0.012) (28000,0.011) (28500,0.011) (29000,0.011) (29500,0.011) (30000,0.011) (30500,0.012) (31000,0.011) (31500,0.012) (32000,0.011) (32500,0.012) (33000,0.012) (33500,0.011) (34000,0.011) (34500,0.012) (35000,0.012) (35500,0.011) (36000,0.011) (36500,0.011) (37000,0.011) (37500,0.010) (38000,0.011) (38500,0.011) (39000,0.011) (39500,0.011) (40000,0.011) (40500,0.011) (41000,0.011) (41500,0.010) (42000,0.010) (42500,0.010) (43000,0.010) (43500,0.011) (44000,0.010) (44500,0.011) (45000,0.010) (45500,0.009) (46000,0.010) (46500,0.010) (47000,0.010) (47500,0.010) (48000,0.009) (48500,0.009) (49000,0.008) (49500,0.010) (50000,0.009) (50500,0.008) (51000,0.009) (51500,0.009) (52000,0.009) (52500,0.009) (53000,0.009) (53500,0.008) (54000,0.009) (54500,0.009) (55000,0.008) (55500,0.008) (56000,0.009) (56500,0.009) (57000,0.009) (57500,0.008) (58000,0.009) (58500,0.008) (59000,0.008) (59500,0.008) (60000,0.008) (60500,0.008) (61000,0.008) (61500,0.007) (62000,0.007) (62500,0.008) (63000,0.008) (63500,0.008) (64000,0.008) (64500,0.007) (65000,0.007) (65500,0.007) (66000,0.008) (66500,0.007) (67000,0.008) (67500,0.007) (68000,0.007) (68500,0.007) (69000,0.007) (69500,0.007) (70000,0.008) (70500,0.007) (71000,0.007) (71500,0.007) (72000,0.007) (72500,0.007) (73000,0.007) (73500,0.007) (74000,0.008) (74500,0.007) (75000,0.007) (75500,0.007) (76000,0.006) (76500,0.006) (77000,0.007) (77500,0.006) (78000,0.006) (78500,0.006) (79000,0.006) (79500,0.006) (80000,0.006) (80500,0.006) (81000,0.006) (81500,0.006) (82000,0.006) (82500,0.006) (83000,0.006) (83500,0.006) (84000,0.006) (84500,0.006) (85000,0.006) (85500,0.006) (86000,0.006) (86500,0.006) (87000,0.006) (87500,0.006) (88000,0.006) (88500,0.006) (89000,0.006) (89500,0.006) (90000,0.005) (90500,0.006) (91000,0.006) (91500,0.005) (92000,0.006) (92500,0.005) (93000,0.005) (93500,0.005) (94000,0.006) (94500,0.006) (95000,0.006) (95500,0.005) (96000,0.005) (96500,0.006) (97000,0.005) (97500,0.006) (98000,0.006) (98500,0.005) (99000,0.005) (99500,0.005) (100000,0.005)
};
\addlegendentry{8};

\end{axis}
\end{tikzpicture}
\end{center}
\end{wrapfigure}
training set and the other half as a test set. We calculated HoG
features (\cite{dalal2005histograms}) from the images\footnote{Using
  the Matlab implementation provided in
  \url{http://www.mathworks.com/matlabcentral/fileexchange/33863-histograms-of-oriented-gradients}.}.
We then trained, using GECO, a depth-$2$ polynomial network on the resulting
features. We used $40$ neurons in the hidden layer. For comparison we trained
the same network architecture (i.e. $40$ hidden neurons with a squared
activation function) by SGD. We also trained a similar network ($40$ hidden
neurons again) with the ReLU activation function.  For the SGD implementation
we tried the following tricks to speed up the convergence: heuristics for
initialization of the weights, learning rate rules, mini-batches, Nesterov's
momentum (as explained in \cite{sutskever2013importance}), and dropout. The
test errors of SGD as a function of the number of iterations are depicted on
the top plot of the Figure on the side. We also mark the performance of GECO
as a straight line (since it doesn't involve SGD iterations). As can be seen,
the error of GECO is slightly better than SGD. It should be also noted that
we had to perform a very large number of SGD iterations to obtain a good
solution, while the runtime of GECO was much faster. This indicates that GECO
may be a valid alternative approach to SGD for training depth-$2$ networks.
It is also apparent that the squared activation function is slightly better
than the ReLU function for this task.

The second plot of the side figure demonstrates the
benefit of over-specification for SGD. We generated random
examples in $\reals^{150}$ and passed them through a random
depth-$2$ network that contains 60 hidden neurons with the ReLU
activation function. We then tried to fit a new network to this data
with over-specification factors of $1,2,4,8$ (e.g., over-specification
factor of $4$ means that we used $60 \cdot 4=240$ hidden neurons).
As can be clearly seen, SGD converges much faster when we over-specify
the network.

\textbf{Acknowledgements:} This research is supported by Intel
(ICRI-CI). OS was also supported by an ISF grant (No. 425/13), and a
Marie-Curie Career Integration Grant.  SSS and RL were also supported
by the MOS center of Knowledge for AI and ML (No. 3-9243). RL is a recipient of the Google Europe Fellowship in Learning Theory, and this research is supported in part by this Google Fellowship.
  We thank Itay Safran for spotting a mistake in a previous
version of \secref{sec:overeasy} and to James Martens for helpful discussions.

\ignore{
\section{Discussion}

In practice, algorithms for training deep neural currently present performances that are state of the art in problems such as vision, speech recognition and more. On the other hand the theory of deep learning seem to be, if any, highly pessimistic. In this paper, we hope to have raised some of the interesting challenges deep learning theory pose. We don't pretend to have solved any of these problems, and we end this paper with more questions than answers...

We've demonstrated that extremely over specified networks are easy to train, and that global optima are ubiquitous in some sense for such networks. 
\textbf{Can the behavior discussed in \secref{sec:overeasy} be extended to
      smaller networks?} 

Using over specification and greedy techniques we've demonstrated that two layer polynomial networks can be learned. These result can be generalized to the task of learning $3$-degree polynomials at a cost that is linear in the dimension. Similarly, one can show that $t$-degree polynomials can in general be learned at time $d^{t-2}$.
\textbf{Is it possible to learn polynomials of bounded degree more
      efficiently? Is it impossible to efficiently learn polynomials of degree
      $\omega(1)$?}

Finally, we would like to have better understanding as to what guarantees can be made with respect to deep learning.
\textbf{What combination of distributional assumptions and network
      structure allows provably efficient deep learning?}
}

\bibliographystyle{plain}
\bibliography{DPN}

\newpage
\appendix
\section{Proofs}

\subsection{Proof of \corref{cor:sigmoidHard}}\label{ap:sigmoidHard}
\subsubsection{Hardness result for the class $\nn_{2,n,\sigma_{\sig},L}$:}

Consider $H^a$ as defined in \thmref{thm:hard}. Note that for every $h\in H^a$ there are integral $\ww$ and $b$ such that $h(\xx)= \ww^{\top} \xx - b-\frac{1}{2}$ and we have that
$|h(\xx)|\ge 1/2$.
Given $k$ hyperplanes $\{h_i\}_{i=1}^k$ consider the neurons
\[g_i(\xx) = 1/\left(1+\exp\left(-C h_{i}(\xx)\right)\right),\]
where $C\in \omega(1)$ is to be chosen later.
  Let
 \[ g(\xx)= \frac{C d}{2k+\frac{1}{3}}\left(\sum_{i=1}^k g_i(\xx) - k+\frac{1}{3}\right).\]

If $\|\ww\|_1+|b+\frac{1}{2}|\le d$ we have that $g(\xx)\in \nn_{2,n,\sigma_{\sig},L}$, whenever $L\ge Cd$.
Choose $C\in O(k)$ sufficiently large so that
 \[1/\left(1+\exp\left(-\frac{C}{2}\right)\right) - 1> -\frac{1}{3 k}.\] and
 \[1/\left(1+\exp\left(\frac{C}{2}\right)\right) < \frac{2}{3}.\]
   Since $|h_i(\xx)|\ge \frac{1}{2}$ for all $i$, if the output of all neurons $\{g_i\}$ is positive we have
   \[\frac{2k+1/3}{Cd} g(\xx) \ge k\left(\frac{1}{1+\exp(-\frac{C}{2})}-1\right)+\frac{1}{3}>0.\]
   On the other hand, if $h_i(\xx)<-\frac{1}{2}$ for some $i$ we have that \[\frac{2k+1/3}{Cd} g(x) \le \sum_{i=1}^{k-1} g_i(\xx) +\frac{1}{1+\exp(\frac{C}{2})} -k +\frac{1}{3}\le k-1 -k +\frac{1}{1+\exp\frac{C}{2}}+\frac{1}{3}< 0.\]
 We've demonstrated that the target function $\mathrm{sign}(g(\xx))$ implements $h_1\wedge h_2\wedge\ldots \wedge h_k$ thus $H_k^a \subseteq \nn_{2,k+1,\sigma_{\sig},L}$.
 \subsubsection{Hardness result for the class $\nn_{2,n,\sigma_{\relu},L}$:}
 Again, given $k$ hyperplanes $\{h_i\}_{i=1}^k$, for every $k$ consider the two neurons:
 \[g^{+}_i(\xx) = \max\{0,2 h_i(\xx)\}, \quad g_i^{-}=(\xx)\max\{0,2 h_i(\xx)-1\}.\]
 And let
 \[g(\xx) =\frac{1}{2k} \left(\sum_{i=1}^k \left(g^{+}_i (\xx)-g_i^{-}(\xx)\right) -k\right).\]
 As before $g(\xx)\in \nn_{2,2k+1,\sigma_{\relu},L}$, whenever $L\ge 2d$. One can also verify that $g(\xx)$ implements $h_1\wedge h_2 \wedge \ldots \wedge h_k$.
\subsection{Proof of \thmref{thm:turingPoly}}\label{ap:turingPolly}

  We start by showing that we can implement $\mathrm{AND},\mathrm{OR},\mathrm{NEG},\mathrm{Id}$ gates 
  using polynomial networks of fixed depth and size. As a corollary, we can implement circuits with fixed number of fan-ins.
  $\mathrm{NEG}$ can be implemented with $x\mapsto 1-x$ and $\mathrm{Id}$ can be implemented with $x\mapsto \frac{1}{4}\left((x+1)^2 - (x-1)^2\right)$. Next note that

\[\mathrm{AND}(\xx_1,\xx_2)= \xx_1\cdot\xx_2, \quad \mathrm{and} \quad
\mathrm{OR}(\xx_1,\xx_2)= \xx_1+\xx_2 - \mathrm{AND}(\xx_1,\xx_2).\]
and that $\xx_1\cdot \xx_2 =\frac{1}{4}\left( (\xx_2+\xx_1)^2 - (\xx_2-\xx_1)^2\right).$
Thus we can implement with two layers a conjunction and disjunction of $2$ neurons. By adding a fixed number of layers, we can also implement the conjunction and disjunction of any fixed number of neurons. Therefore, if $B$ is a circuit with fixed number of fan-ins, of size T, we can implement it using a polynomial network with $O(T)$ layers and $O(T^2)$ neurons, where layer $t$ simulates the calculation of all gates at depth $t$.

Now, by \cite{pippenger1979relations}, any Turing machine with runtime
$T$ can be simulated by an \emph{oblivious} Turing machine with
$O(T\log T)$-steps. An oblivious Turing machine is a machine such that
the position of the machine head at time $t$ does not depend on the
input of the machine (and therefore is known ahead of time). We can
now easily simulate the machine by a network of
depth $O(T \log T)$, where the nodes at each layer contain the state
of the turing machine (the content of the tape and the position at the
state machine), and the transition from layer to layer depends on a
constant size circuit, and hence can be implemented by a constant
depth polynomial network.

\subsection{Proof of \thmref{thm:ntop}}\label{ap:ntop}
The idea of proof of \thmref{thm:ntop} is as follows: First we show that we can express any $T$-degree polynomial using $O(\log T)$ layers and $O(T)$ neurons. This is done in \lemref{lem:polyexpressive} part \ref{poly}. As a second step, we show in \lemref{lem:sigmoid} that a sigmoidal function can be approximated in a ball of radius $L$ by a
$O(\log\frac{L}{\epsilon})$-degree polynomial. The result follows by replacing each sigmoid activation unit with added layers that approximate the sigmoidal function on the output of the previous layer. We will first prove the two Lemmas. The proof of \thmref{thm:ntop} is then given at the end of the section.

\begin{lemma}\label{lem:polyexpressive} The following statements hold:

\begin{enumerate}
\item\label{id} If $g\in \nn_{t,n,\sigma_2,L}$ for some $L\ge 2$ then $g\in \nn_{t',n+2(t'-t),\sigma_2,L}$ for every $t'\ge t$.
\item\label{prod} If $G \in \nn_{t,n,\sigma_2,L}$ for some $L\ge 2$ and $G$ is a network with two output neurons $g_1$ and $g_2$ then $g_1 \cdot g_2 \in \nn_{t+1,n+1,\sigma_2}$.
\ignore{
  Every target function of the form $p(\xx)=(\ww_1^\top \xx)\cdot (\ww_2^\top \xx)$ is in $\nn_{3,3,\sigma_2}$.

Further, if $\|\ww_1\|_1,\|\ww_2\|_1 \le L$ for some $L\ge 2$ then $p$ is in $\nn_{2,3,\sigma_2,L}$.}
\item\label{power} If $g\in \nn_{t,n,\sigma_2,L}$ for some $L\ge 2$ then $(g)^T \in \nn_{t',n',\sigma_2,L}$. where
$t'=t+\log T+ \log\log T$ and $n'=n+2\log T + \log T( \log\log T)$.
\ignore{
Every target function of the form $m(\xx)=(\ww^\top \xx)^T$ is in $\nn_{t,n,\sigma_2}$ where $t=\log T+ \log\log T$ and $n= 2\log T + \log T( \log\log T)$.

Further, if $\|\ww\|_1\le L$ for some $L\ge 2$ then $m$ is in $\nn_{t,n,\sigma_2,L}$.}
\item\label{poly} If $g\in \nn_{t,n,\sigma_2,L}$ then $\sum_{i=1}^T a_i (g(\xx))^k$ is in $\nn_{t',n',\sigma_2,L'}$ where
\[t'=t+\log T+\log\log T, \quad n'=n+2\|a\|_0(2\log T+\log T(\log\log T),\] where $\|a\|_0=|\{k:a_k\ne 0\}|$. And $L'=\max\{\|a\|_1,L,2\}$.
\ignore{
  Every target function of the form $q(\ww^\top \xx) = \sum_{k=1}^T a_i (\ww^\top\xx)^k$ is in $\nn_{t,n,\sigma_2}$ where $t=\log T+\log\log T$ and $n=2\|a\|_0(2\log T+\log T(\log\log T)$, where $\|a\|_0=|\{k:a_k\ne 0\}|$.

Further, if $\|a\|_1,\|\ww\|_1 \le L$ for some $L\ge 2$ then $q$ is in $\nn_{t,n,\sigma_2,L}$.}
\end{enumerate}

\end{lemma}
\begin{proof}
\begin{enumerate}
\item Proof of \ref{id}]: Note that $\frac{1}{4}((x+1)^2 -(x-1)^2)=x$. Next we prove the statement by induction. For $t'=t$, the satement is trivial. Next assume that $g\in \nn_{t,n+2(t'-t),\sigma_2,L}$. Let
\[ h_1(\xx) = \left(\frac{1}{2} g(\xx)+\frac{1}{2}\right)^2, \quad h_2(\xx)=\left(\frac{1}{2} g(\xx)-\frac{1}{2}\right)^2.\]
Let $h(\xx)= h_1(\xx)- h_2(\xx)$ then $h(\xx)=g(\xx)$. By taking the network that implements $g$, removing the output neuron, adding an additional hidden layer that consists of $h_1$ and $h_2$ and finally adding an additional output neuron we have that $h(\xx)\in \nn_{t'+1,n+2(t'-t)+2,\sigma_2,L}$.
\item Proof of \ref{prod}: Like before, note that $x_1\cdot x_2= \frac{1}{4}(x_1+x_2)^2 - \frac{1}{4}(x_1-x_2)^2$. Let
\[ h_1(\xx)= (\frac{1}{2}g_1(\xx)+\frac{1}{2}g_2(\xx))^2, \quad h_2(\xx) = (\frac{1}{2}g_1(\xx)+\frac{1}{2}g_2(\xx))^2.\]
As before we remove from the network that implements $G$ the two nodes at the output layer, add an additional hidden layer that implements $h_1$ and $h_2$ and finally add an output neuron $h(\xx)=h_1(\xx)-h_2(\xx)$.

\item Proof of \ref{power}:Write $T= \sum_{i=1}^{\log T} \epsilon_i 2^i$ where $\epsilon_i=\{0,1\}$.

We will first show that we can construct a polynomial network that contains in layer $t+\log T$  neurons $h_1,\ldots,h_{\log T}$ such that $h_k(\xx)=(g(\xx))^{2^k}$. It is easy to see that we can implement a neuron $h'(\xx)_k$ at layer $t+k$ such that $h_k'(\xx)=(g(\xx))^{2^k}$.  Next, using \ref{id} we add $2(\log T-k)$ neurons and implement $h_k'$ in layer $t+ \log T$.

Finally, we implement $\prod_{\{i:\epsilon_i \ne 0\}}h_i(\xx) $  using $\log \log T$ layers and $\log T \log\log T$ additional neurons, this can be done by applying \ref{prod} where at each layer we pair the neurons at previous layer and do their product (e.g. if for every $i$ $\epsilon_i \ne 0$ then at the next layer we implement ($h_1\cdot h_2, h_3\cdot h_4,\ldots h_{t-1} h_t$) then at the next layer we implement ($h_1\cdot h_2\cdot h_3\cdot h_4,\ldots, h_{t-4}\cdots h_t$) etc..)
\item Proof of \ref{poly}: Follows from \ref{id} and \ref{power}.
\end{enumerate}
\end{proof}
\begin{lemma}[Sigmoidals are approximable via polynomial networks]\label{lem:sigmoid}
The following holds for any $\epsilon \ge 0$ and (for simplicity) $L\ge 3$:
Set
\[ T= \log \left(2L^4 + \exp\left(7L \ln\left(\frac{4L}{\epsilon} +3\right)\right)\right)+2\log\frac{8}{\epsilon}.\]
There is a polynomial $p(x)=\sum_{j=1}^T a_j x^j$, such that:

\[ \sup_{|x|<4L} |p(x)-\sigma_\sig(x)|<\epsilon.\]
\ignore{and
\begin{equation} t=1+ \log T + \log\log T \in \tilde{O}\left(\log L \log \frac{1}{\epsilon}\right), \end{equation}
\begin{equation}
n=1+ 2 T\left(2\log T + \log T \log \log T\right)\in \tilde{O}(L\log \frac{1}{\epsilon}).\end{equation}
Then for every $\ww\in\reals^d$ there is $p\in \nn_{t,n,\sigma_2}$ such that:
\[ \sup_{\xx:|\ww^\top\xx|<4L} |p(\xx)-\sigma_{\sig}(\ww^\top \xx)|\le \epsilon.\]}
\end{lemma}
\begin{proof}[Proof of \lemref{lem:sigmoid}]
Set
\[ t'= \log \left(2L^4 + \exp(7L \ln\left(\frac{4L}{\epsilon} +3\right)\right).\]
According to \cite{Shamir2010Learning} Lemma 2, there is an analytic function $q$ such that
\[ \sup_{|x| \le 1} |q(x)-\sigma_{\sig}(4Lx)| \le \frac{\epsilon}{2},\]
and
\[q(x)= \sum_{j=0}^\infty \beta_j x^j\] where
\[ \sum_{j=0}^\infty \beta_j^2 2^j \le 2^{t'}.\]

Note that for every $j$ we have  $|\beta_j| \le 2^{\frac{t'-j}{2}}$. Thus
\[\sup_{|x|<1} |\sum_{j > T} \beta_j x^j | \le   \sum_{j>T} |\beta_j| \le \sum_{j>T}  2^{\frac{t'-j}{2}}=2^{\frac{t'-T}{2}}\sum_{j=1}^\infty \left(\sqrt{2}\right)^{-j}< 4 \cdot 2^{\frac{t'-T}{2}}.\]
Recalling that $T=t'+2\log\frac{8}{\epsilon}$ and letting $p_0(x)=\sum_{j=0}^T \beta_j x^j$, we have by triangular inequality that
\[ \sup_{|x| \le 1} |p_0(x)-\sigma_{\sig}(4L x)| \le \epsilon.\] Finally, take $p(x)=p_0(\frac{x}{4L})$.
\end{proof}
\subsubsection{Back to proof of \thmref{thm:ntop}}
Set
\[T= \log \left(2L^4 + \exp(7L \ln\left(\frac{(4L)^{t}}{\epsilon} +3\right)\right)+2\log\frac{8(4L)^{t-1}}{\epsilon}.\]
and have
\begin{equation}\label{eq:B_t} B_t=1+ \log T + \log\log T \in \tilde{O}\left(\log L \log \frac{L^t}{\epsilon}\right), \end{equation}
\begin{equation}\label{eq:B_n}
B_n=1+ 2 T\left(2\log T + \log T \log \log T\right)\in \tilde{O}(L\log \frac{L^t}{\epsilon}).\end{equation}

We prove the statement by induction on $t$, our induction hypothesis will hold for networks with not necessarily a single output neuron. For $t=1$, since $\nn_{1,n,\sigma_2}=\nn_{1,n,\sigma_\sig}$, the statement is trivial.
\ignore{
The exact induction hypothesis is as follows, for every $\epsilon$ and $L$ set:
\[B(t)=\log\left( \log \left(2L^4 + \exp(7L \ln\left(\frac{4(4L)^{t+1}}{\epsilon} +3\right)\right)+2\log\frac{2(4L)^t}{\epsilon}\right),\]
For every $F\in \nn^L_{t,n}$ there is $P\in \dpn^L_{tB(t),2nB(t)}$ such that:

\[\sup_{\|\xx\|_{\infty}<1}\|F(\xx)-P(\xx)\|_{\infty}<\epsilon\]
}
Next let $F\in \nn_{t,n,\sigma_\sig,L}$, assume $F:\mathbb{R}^d\to \mathbb{R}^s$ (i.e. the output layer has $s$ nodes). There is a target function $F^{(t-1)} \in \nn_{t-1,n-s,\sigma_\sig,L}$ such that for every $i=1\ldots s$ we have
\[F_i = \ww^\top_i \sigma_{\sig} (F^{(t-1)}(\xx)).\]
where $\sigma_{\sig} (F^{(t-1)}(\xx))$ denotes pointwise activation of $\sigma_{\sig}$ on the coordinates of $F^{(t-1)}$.

By induction, there is some $P^{(t-1)}\in \nn_{(t-1)B_t,(n-s)B_n,\sigma_2}$ such that

\[\sup_{\|\xx\|_\infty\le 1}\| P^{(t-1)}(\xx)-F^{(t-1)}(\xx)\|_\infty \le \frac{\epsilon}{4L}\le \frac{\epsilon}{4}\]
By \lemref{lem:polyexpressive} part \ref{poly} and \lemref{lem:sigmoid} we can add $B_t$ layers and $B_n$ neurons and implement a new target function $P_i$ such that

\[P_i(\xx)= \ww_i^\top p(P^{(t-1)}(\xx)),\] where $p$ is taken from Lemma \ref{lem:sigmoid} and satisfies
\[ \sup_{|x|\le 4L} |p(x)-\sigma_{\sig}(x)|<\frac{\epsilon}{(4L)^{t-1}}\le \frac{\epsilon}{2L},\]

  Taken together we can add $s B_n$ nodes to implement a target function $P=P_1,\ldots, P_s$.
Next,

\[\|P(\xx)-F(\xx)\|_\infty\le \sup_{i} \| P_i(\xx)-\ww_i^\top \sigma_\sig(P^{(t-1)}(\xx))\|+ \|\ww_i^\top \sigma_\sig(P^{(t-1)}(\xx))-\ww_i^\top \sigma_\sig(F^{(t-1)}(\xx))\|.\]
Recall that the $\ell_1$-norm of each weight vector of each neuron is bounded by $L$ and that the output of each neuron is bounded by $\sup_{x} \sigma_\sig(x) =1$, hence: $\|F^{(t-1)}(\xx)\|_\infty\le L$. By induction we also have that $\|F^{(t-1)}(\xx)-P^{(t-1)}(\xx)\|_\infty \le 1$ hence $\|P^{t-1}(\xx)\|_\infty \le 2L$ and we have:
\[\sup_{i} \| \ww_i^\top p(P^{(t-1)}(\xx))-\ww_i^\top \sigma_\sig(P^{(t-1)}(\xx))\|+ \sup_{i}\|\ww_i^\top \sigma_\sig(P^{(t-1)}(\xx))-\ww_i^\top \sigma_\sig(F^{(t-1)}(\xx))\|\le.\]

\[ \frac{\|\ww_i\|_1\epsilon}{2L} +\|\ww_i\|_1\|P^{(t-1)}(\xx)-F^{(t-1)}(\xx)\|_\infty\le \frac{\epsilon}{2}+\frac{\epsilon}{4}\le\epsilon.\]
Where we used the fact that $\sigma_\sig$ is $1$-Lipschitz.
\subsection{Proof of \thmref{thm:3Gec}}\label{ap:3Gec}
That $f\in \nn_{3,5r,\sigma_2}$ can be shown using \lemref{lem:polyexpressive} and the output's structure.

Let us denote by $\Approx(\frac{(1-\tau)}{\sqrt{2d}},\nabla R(f))$, a procedure that returns $g \in \mathcal{V}$ such that

\[\sum_{i=1}^m \ell'(f(\xx_i),y_i) g(\xx_i) \ge \frac{(1-\tau)}{\sqrt{2d}} \max_{g^*\in \mathcal{V}}\frac{1}{m} \sum_{i=1}^m \ell'(f(\xx_i),y_i) g^*(\xx_i)\]
\begin{figure}[h]
\begin{center}
\fbox{
\begin{Balgorithm}
Input: $r$ $\tau$, $\epsilon$\\
Initialize: $\mathcal{W}=\emptyset$, $f=0$\\

  \textbf{For} $t=1,\ldots ,r$\+\+\\
  Set $g(\xx):=\Approx\left(\frac{(1-\tau)}{\sqrt{2d}},\nabla R(f)\right)$\\
  Add $g(\xx)$ to $\mathcal{W}$.\\
  Let $\alpha^{(t)}$ and $f^{(t)}$ optimize the problem $\min_{f} \frac{1}{m}\sum_{i=1}^m \ell(f(\xx_i),y_i) $\\ subject to $f(\xx)=\sum_{g\in\mathcal{W} }\alpha_g g (\xx)$.\\
  \-\-\\
Return: $f=f^{(r)}$.
\end{Balgorithm}}\caption{GECO with different $\Approx$ procedure.}\label{alg:3Gec}
\end{center}
\end{figure}

The GECO algorithm is presented in \figref{alg:3Gec} with an $\Approx$ procedure that is implemented with respect to $\mathcal{V}=\cup \mathcal{V}_i$. The guarantees in \thmref{thm:3Gec} are proven in exactly the same manner as in \cite{shalev2011large}.

The remained challenge is to demonstrate that the $\Approx$ procedure can be implemented efficiently, relying on the algorithm presented in \figref{fig:tensor} . To this end, note that the only difficulty is when $g^* \in \mathcal{V}_3$ (since if $g^*\in \mathcal{V}_2$ or $g^*\in \mathcal{V}_1$ we are back to the $2$-layer scenario). The proof follows directly from the following lemma:
\begin{lemma}\label{lem:tensorapprox}
Let $\ww^*,\uu^*,\vv^*$ be the output of the Algorithm presented in \figref{fig:tensor} with parameters $\delta,\tau,\{\xx_i\}_{i=1}^m$ and $\alpha_i = \ell'(f(\xx_i),y_i)$.
With probability at least $1-\delta$:
\[F(\ww^*,\uu^*,\vv^*)\ge \frac{1-\tau}{\sqrt{2d}}\max_{\|\ww\|,\|\uu\|,\|\vv\|\le 1} F(\ww,\uu,\vv),\]
where
\[F(\ww,\uu,\vv)= \frac{1}{m} \sum_{i=1}^m \ell'(f(\xx_i),y_i) (\ww^\top \xx_i)\cdot (\uu^\top \xx_i)\cdot (\vv^\top \xx_i).\]

\end{lemma}

\begin{proof}

Let us denote by $\ww^*,\uu^*,\vv^*$ the maximizers of $F(\ww,\uu,\vv)$, over all $\|\ww\|,\|\uu\|,\|\vv\|=1$.

For each $\uu,\vv$ let $f(\uu,\vv)$ be the vector
\[f(\uu,\vv)= \sum_{i=1}^m \alpha_i  (\uu^\top \xx_i) (\vv^\top \xx_i)\xx_i .\]

First, we claim that $f(\uu^*,\vv^*) \propto \ww^*$ and that $F(\ww^*,\uu^*,\vv^*)=\|f(\uu^*,\vv^*)\|$. Indeed for every $\|\ww\|\le 1$, by the Cauchy-Schwartz inequality:
\[ F(\ww,\uu^*,\vv^*)= f(\uu^*,\vv^*)^\top \ww \le \|f(\uu^*,\vv^*)\|\|\ww\|\le \|f(\uu^*,\vv^*)\|.\]
Again by Cauchy-Schwartz, equality is attained if and only if $\ww \propto f(\uu^*,\vv^*)$.

Next, let us consider a single random variable $\hat{\ww}$ such that $\ww \sim N(0,Id)$ and $\hat{\ww}=\frac{\ww}{\|\ww\|}$. Note that for any unit vector $\uu_1$ we have $\mathbb{E}((\hat{\ww}^\top \uu_1)^2)=\frac{1}{d}$. Indeed, extend $\uu_1$ to an orthonormal basis $\uu_1,\ldots, \uu_d$. we have that
\[1= \mathbb{E}(\|\hat{\ww}\|^2)=\mathbb{E}(\sum_{i=1}^d (\hat{\ww}^\top \uu_i )^2).\] By symmetry we have that:
\[1=\mathbb{E}(\sum_{i=1}^d (\hat{\ww}^\top \uu_i )^2)=\sum_{i=1}^d\mathbb{E}( (\hat{\ww}^\top \uu_i )^2) =d\mathbb{E}((\hat{\ww}^\top \uu_1)^2).\]

In particular we have  $\mathbb{E}((\hat{\ww}^\top \ww^*)^2)=\frac{1}{d}$. In conclusion $(\hat{\ww}^\top \ww^*)^2$ is a random variable that takes values in $[0,1]$ and has expected value $\frac{1}{d}$. Applying the inverse Markov inequality (i.e. applying Markov to the random variable $1-(\hat{\ww}^\top \ww^*)^2$), we have that

\[P((\hat{\ww}^\top \ww^*)^2 > \frac{1}{2 d}) \ge \frac{\frac{1}{d} -\frac{1}{2d}}{1-\frac{1}{2d}}=\frac{1}{2d-1}\in O(\frac{1}{2d})\]

Letting $s\ge -\frac{\log\frac{1}{\delta}}{\log(1-\frac{1}{2d})}\approx 2d\log \frac{1}{\delta}$ we have that with probability at least $(1-\delta)$ for some $\ww_i$ we have $|\ww_i^\top \ww^*|\ge \frac{1}{\sqrt{2 d}}$, say $\ww_1$.

Finally note that by definition of $\uu_i,\vv_i$:

\[\max_{i\le s} F(\ww_i,\uu_i,\vv_i)\ge F(\ww_1,\uu_1,\vv_1) \ge (1-\tau)\max_{\uu,\vv} F(\ww_1,\uu,\vv) \ge
(1-\tau) F(\ww_1,\uu^*,\vv^*)=\]
\[=(1-\tau)\|f(\uu^*,\vv^*)\| {\ww^*}^\top\ww_1\ge\frac{1-\tau}{\sqrt{2 d}}f(\uu^*,\vv^*)=\frac{1-\tau}{\sqrt{2d}} F(\ww^*,\uu^*,\vv^*)\]

\end{proof}

\end{document}